\documentclass{article}

     \usepackage[preprint]{neurips_2024}

\usepackage[utf8]{inputenc} % allow utf-8 input
\usepackage[T1]{fontenc}    % use 8-bit T1 fonts
\usepackage{hyperref}       % hyperlinks
\usepackage{url}            % simple URL typesetting
\usepackage{booktabs}       % professional-quality tables
\usepackage{amsfonts}       % blackboard math symbols
\usepackage{subfigure}
\usepackage{amsmath}
\newtheorem{theorem}{Theorem}
\newtheorem{lemma}{Lemma}
\newtheorem{proof}{Proof}[section]
\newtheorem{property}{Property}
\usepackage{makecell}
\usepackage{multirow}
\usepackage{nicefrac}       % compact symbols for 1/2, etc.
\usepackage{microtype}      % microtypography
\usepackage{xcolor}         % colors
\usepackage{algorithm}
\usepackage{algorithmicx}
\usepackage{wrapfig}

\usepackage{graphicx}
\usepackage{tabularx}
\usepackage{algpseudocode}

\title{ContraSolver: Self-Alignment of Language Models by Resolving Internal Preference Contradictions}

\author{%
  Xu Zhang\footnotemark[1],~~Xunjian Yin\footnotemark[1], ~~Xiaojun Wan \\
  Wangxuan Institute of Computer Technology, Peking University\\
  \texttt{\{zhangxu, xjyin, wanxiaojun\}@pku.edu.cn} \\
}

\begin{document}

\maketitle

\renewcommand{\thefootnote}{\fnsymbol{footnote}}
\footnotetext[1]{These authors contributed equally to this work.}
% \footnotetext[2]{Corresponding author.}
\renewcommand*{\thefootnote}{\arabic{footnote}}

% contribution：
% 1. preference graph； 2. method（generation） 3. result（preference&consistency）
% sec3、4: math，BT model 介绍 graph（普遍存在）。；介绍数据构建&方法。
% experiment：
% 1. triple ring 个数。（分析折线图；方法的自适应选环？）
% 2. 选边来dpo。（highest 作为我们的方法；其他的作为消融实验。）
% 3. 4x4x4
% 4. 
% sec 5 实验
% ·（具体data构建）
% 分析。内部一致性、环消解率、之类的。

% checklist errorbar 和 broader impact、safety
\begin{abstract}
While substantial advancements have been made in developing large language models (LLMs), achieving control over their behavior can be difficult. 
Direct preference optimization (DPO) assumes the existence of a latent reward function to evaluate the responses of LLMs.
This assumption indicates a strict preference ordering of different responses to the same input.
However, there always exist contradictions of preference in LLMs according to our experimental observations. 
In this paper, we construct a graph structure of the preference relationship among different responses with self-annotation to find contradictions in the preference order.
We propose ContraSolver, an algorithm that traverses all edges on the preference graph to identify those that might cause contradictions.
ContraSolver initializes the graph with a maximum spanning tree and identifies contradictory edges, prioritizing the resolution of low-confidence preferences while preserving high-confidence ones.
Experimental results on four different generation tasks show that the performance of different LLMs can be largely improved through our completely unsupervised self-alignment.
Furthermore, by analyzing the preference graphs of LLMs with and without self-alignment by ContraSolver, we quantify the reduction in contradictions, suggesting that resolving preference contradictions is crucial for achieving better alignment performance.

\end{abstract}

\section{Introduction}
Large language models (LLMs) \citep{brown2020language,thoppilan2022lamda,bubeck2023sparks} have significantly advanced natural language understanding and generation. 
However, LLMs can generate undesirable outputs \citep{gehman-etal-2020-realtoxicityprompts,carlini2021extracting,abid2021persistent} and unfaithful content \citep{ji2023survey,zhou2023navigating}, posing security risks in NLP applications. 
Thus, aligning LLMs with human values is essential \citep{shen2023large}. Reinforcement learning from human feedback (RLHF) has become the primary approach for aligning LLMs \citep{ouyang2022training,bai2022training}. 
Given the challenge of obtaining extensive, consistently high-quality manual annotations \citep{lee2023rlaif}, some studies propose self-alignment \citep{NEURIPS2023_0764db11, yuan2024self,zhang2024self}, creating pairwise preference data autonomously from the LLM itself.

Typical algorithms for fine-tuning LLMs with preference data in RLHF, such as Proximal Policy Optimization (PPO) \citep{schulman2017proximal} and Direct Preference Optimization (DPO) \citep{rafailov2024direct}, assume a latent reward function that dictates preference between responses.
The preference between two different responses is generated according to the values given by the reward function.
This assumption implies a strict ordering relationship among different responses.
Therefore, given an input instruction and a group of responses, a perfectly aligned LLM is supposed to have a strict preference ordering over these responses.
Unfortunately, we observe LLMs often exhibit preference contradictions in experiments.

Existing self-alignment methods \citep{gulcehre2023reinforced, guo2024human} train LLMs on pairwise preference data generated by themselves to align models.
They often require large amounts of training data and multiple iterations based on seed instruction-tuning data or handcrafted principles \citep{huang2022large, zhang2023self}.
These approaches indiscriminately trust the model's preference labeling of all input responses, regardless of the model's proficiency with these inputs.
However, based on the observations in our experiments, model labeling preferences are frequently inconsistent and unreliable.
While effective for certain tasks, these methods overlook LLMs' weaknesses and inconsistencies.  This potentially limits the overall performance of aligned LLMs and weakens the motivation of self-alignment.

As preference labels in self-alignment are generated without human intervention, it is important to select appropriate data for training.
The selected data pairs are supposed to be both informative and reliable.
To address these concerns, in this paper, we focus on eliminating the inconsistencies to better align the LLM by eliminating internal contradictions. 
We construct \textbf{preference graph} for LLMs, where vertices represent possible responses and edges denote pairwise preferences to better study the preference relationship. 
The preference contradictions in LLMs can be modeled as cycles in the preference graph.
For perfectly-aligned LLMs, the preference graph is supposed to be a directed acyclic graph.

To enhance consistency in LLMs, we propose a novel self-alignment algorithm, ContraSolver to resolve the inconsistent preferences. 
ContraSolver leverages the model's confidence scores to systematically identify and resolve contradictions in its preference graph. 
The algorithm initializes the graph with a max-weight tree, and iteratively identifies contradictory edges and adds correct edges in reverse loops and forward loops.
By traversing the preference graph, ContraSolver filters out contradictory edges and selects edges that are conducive to eliminating contradictions (which we name as heuristic edges).
We prove that ContraSolver ensures the weight of the identified contradictory edge is always lower than the heuristic edges that have contradictions with it.
By performing DPO training on selected pairwise preferences, ContraSolver refines the preference graph to make it globally consistent.
To the best of our knowledge, we are the first to work on LLM self-alignment based on the internal preference consistency of models. 

Experimental results on four text generation tasks demonstrate the superiority of our approach in enhancing LLM performance. 
We further analyze the difference between preference graphs before and after self-alignment with ContraSolver.
The results illustrate a significant reduction in the number of contradictions on preference graphs after self-alignment with ContraSolver.
This phenomenon confirms our motivation that resolving contradictions is conducive to improving the internal consistency and the performance on generation tasks of LLMs.

In summary, we make the following key contributions in this work:
\begin{enumerate}
    \item We prioritize the elimination of inconsistencies in the model's internal preferences, shifting focus from merely reinforcing existing strengths.
    \item We propose ContraSolver, a novel self-alignment algorithm that leverages the principles of maximum spanning trees and confidence-driven edge incorporation which can make the entire preference graph globally consistent.
    \item We present an in-depth analysis and empirical evaluation, demonstrating the effectiveness of ContraSolver in reducing cycles in preference graphs and improving the generation abilities of LLMs.
\end{enumerate}
%With the instructions in Safe-Edit \citep{}, hex-phi \citep{} and Alpaca \citep{}, we sample| directly targets and resolves contradictions in the model's preference graph

% 第2页开头，图1，一条链的偏好结构和实际的图结构对比

%The Bradley-Terry (BT) model \citep{bradley1952rank} is the most commonly used reward function to help model the preference probability \citep{rosset2024direct}.

\section{Background}
This section first revisits Direct Preference Optimization and Bradley-Terry model (Section \ref{sec:dpo}). 
We then highlight that under the DPO assumption, there exists a latent reward function for each response. The assumption indicates a complete ordering in different responses of the same input (Section \ref{sec:reward}).

\subsection{Direct Preference Optimization}
\label{sec:dpo}
Direct Preference Optimization \citep{rafailov2024direct} is proposed to align LLM with human preference.
Prompted with $x$, an LLM autoregressively generates the response $y$ based on the output distribution:
\begin{align}
    \pi_\theta(y|x) = \prod_{i = 1} ^ {n} \pi_\theta(y_i|x, y_{1:n-1}), 
\end{align}
To align $\pi_{\theta}(\cdot|x)$ with human values on preference data, DPO adopts the Bradley-Terry (BT) model \citep{bradley1952rank} for pairwise comparison:
\begin{align}
    p(y_1 \succ y_2 | x) = \frac{\exp(r(x, y_1))}{\exp(r(x, y_1))+\exp(r(x, y_2))}
    \label{eq:bt}
\end{align}
Given $\pi_{\theta}(\cdot|x)$, the LLM to be trained, and a reference model $\pi_{ref}(\cdot|x)$, DPO optimizes the same objective as prior RLHF algorithms:
\begin{align}
    \max_{\pi_\theta} {\mathbb{E}_{x\sim\mathcal{D}, y\sim\pi_{\theta}(\cdot|x)}[r(x, y)-\beta D_{KL}(\pi_{\theta}(\cdot|x)||\pi_{ref}(\cdot|x))]}, 
    \label{eq:dpo}
\end{align}
where $\mathcal{D}$ represents human preference dataset, $r(x, y)$ denotes the reward function and $\beta$ is the coefficient for the reverse KL divergence penalty.
Based on Equation \ref{eq:bt} and Equation \ref{eq:dpo}, DPO eliminates the reward modeling phase by deriving a mapping between the optimal reward function $r^*(x, y)$ and the optimal generation policy $\pi^*_\theta(\cdot|x)$.

\subsection{Latent Reward Function}
\label{sec:reward}
In this paper, we adopt DPO as the training algorithm on pairwise preference data.
As discussed in DPO and some following work \citep{rafailov2024direct,yuan2024self,zhang2024self}, large language models (LLMs) can be utilized as latent reward functions $\mathcal{F}_r$ for evaluating the quality of responses.
Consider an input query $x$ and a set of candidate responses ${y_1, y_2, \ldots, y_n}$. The Bradley-Terry model assumes that any pair of responses $y_i$ and $y_j$ are comparable, and they satisfy a partial order relation. Under this model, $\mathcal{F}_r$ assigns a probability $\mathcal{P}_r(y_i > y_j | x)$ representing the preference of response $y_i$ over $y_j$ given the input $x$.

\section{ContraSolver: Tree-Based and Confidence-Driven Contradiction Resolver} %
\label{sec:graph}
\subsection{Preference Graph Representation}
Given the input query $x$, the model's preferences over the set of responses can be represented as a \textbf{preference graph} $\mathcal{G}= (V, E)$, where $V$ represents the set of responses $\{y_1, y_2, \ldots, y_n\}$, and $E$ represents the set of directed edges between responses. We define an indicator function $I(y_i > y_j|x)$ as follows:
\begin{equation}
I(y_i > y_j|x) = \begin{cases}
1, & \text{if } \mathcal{P}_r(y_i > y_j | x) > 0.5 \\
0, & \text{otherwise.}
\end{cases}
\end{equation}
A directed edge $(y_i, y_j) \in E$ exists if and only if $I(y_i > y_j|x) = 1$. In other words, an edge from $y_i$ to $y_j$ indicates that the model prefers response $y_i$ over $y_j$ for the given input $x$, as determined by the latent reward function $\mathcal{F}_r$.
We assign a weight to edge $(y_i, y_j)$ with the probability of preference:
\begin{equation}
w(y_i, y_j) =  \mathcal{P}_r(y_i > y_j | x).
\end{equation}

For an LLM to be considered consistent in its preferences, the preference graph $\mathcal{G}$ should be acyclic, which can be formally expressed as:
\begin{equation}
\forall i_1, i_2, \ldots, i_k \in {1, 2, \ldots, n}: \quad \neg \left( I(y_{i_1} > y_{i_2} | x) \wedge I(y_{i_2} > y_{i_3} | x) \wedge \ldots \wedge I(y_{i_k} > y_{i_1} | x) \right)
\end{equation}
This equation states that for any sequence of responses $y_{i_1}, y_{i_2}, \ldots, y_{i_k}$, it is not the case that $y_{i_1}$ is preferred over $y_{i_2}$, $y_{i_2}$ is preferred over $y_{i_3}$, and so on, with the last preference being $y_{i_k}$ over $y_{i_1}$, thus forming a cycle. 
Ensuring the acyclicity of the preference graph guarantees that the model's preferences are consistent and transitive, which is a desirable property for reliable evaluation and decision-making.
In this way, we model inconsistencies in model preferences as cycles on the preference graph.

\begin{figure*}
    \vspace{-20pt}
    \centering
    \includegraphics[width=0.95\linewidth]{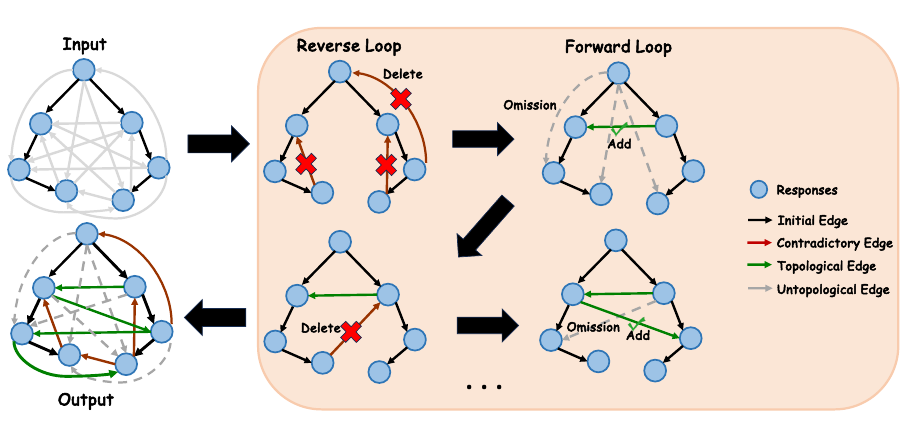}
    \caption{Detailed illustration of the process of ContraSolver traversing the preference graph. The graph is initialized with a maximum spanning tree. In the reverse loop, ContraSolver finds all edges that elicit contradictions to existing edges. In the forward loop, the algorithm omits untopological edges and adds the first topological edge to the graph. Finally, we obtain the heuristic edges(that is bold) for training.}
    \label{fig:alg}

\end{figure*}

%\subsection{Identifying Preference Inconsistencies}
\subsection{Assumption and Motivation} 
As shown in Table \ref{tab:circle_cnt}, our empirical studies reveal that the preference graphs $\mathcal{G}$ of LLMs often contain cycles, implying the existence of inconsistencies in the models' internal preferences. 
We posit that these inconsistencies are crucial, and their resolution could contribute to improved model alignment and capability. 
By introducing the preference graph as a tool, we reformulate the key problem as selecting optimal edges from the preference graph to construct pairwise preference data for training.
We propose an algorithm called ContraSolver for model self-alignment, aimed at enhancing the reliability and trustworthiness of the models.
ContraSolver assumes that edges with higher weights are more reliable.
Different from considering isolated circles on the graph, our proposed algorithm selects edges on the graph from a global perspective by traversing the whole graph and identifying contradictory edges.
By reversing the direction of these edges, all cycles in the graph can be eliminated and the preference graph can form a directed acyclic graph.

\subsection{Methodology}
ContraSolver traverses all edges on the preference graph to divide them into three classes: \textbf{definite edges}, \textbf{contradictory edges}, and \textbf{heuristic edges}.
\textbf{Definite edges} do not appear in any cycle in the preference graph.
We believe that the model is consistent with such preference.
\textbf{Contradictory edges} and \textbf{heuristic edges} appear in at least one of the cycles in the preference graph.
Contradictory edges are those preference pairs that we believe the LLM makes the wrong judgment, denoted as $E_{c}$.
Heuristic edges are those that can help reverse the contradictory edges, thus eliminating circles on the graph, which is denoted as $E_{h}$.
During the traversal process, each edge is chosen to be added to the remaining graph or removed from it.
Formally, given the graph $\mathcal{G}= (V, E)$, we use $E_{r}$ to represent the remaining candidate edge set that has not been visited, and $\mathcal{G'}= (V, E')$ to denote the graph composed of currently retained edges.
It should be noted that $\mathcal{G'}$ is a dynamic graph that is updated in the process of ContraSolver.
For the current graph $\mathcal{G'}$, if a new directed edge $e=(y_i, y_j)$ makes $y_j$ reachable to $y_i$, we consider edge $e$ introduces new topological relationship.
While if $y_i$ can reach $y_j$ on the original graph $\mathcal{G'}$, we take this edge as untopological.
We aim at constructing the contradictory edge set $E_{c}$ that causes contradictions and the heuristic edge set $E_{h}$ that helps solve the contradictions.
Figure \ref{fig:alg} illustrates the process of ContraSolver.

\paragraph{Maximum Spanning Tree Construction}
The first step of the ContraSolver algorithm is to construct a maximum spanning tree $T_0$ from the preference graph $\mathcal{G}$ to initialize $\mathcal{G'}$ using Kruskal algorithm \citep{kruskal1956shortest}. 
This initialization ensures that graph $\mathcal{G'}$ is connected and edges in $\mathcal{G'}$ correspond to the most confident preferences.
After constructing the initial spanning tree $T_0$, ContraSolver iteratively performs the reverse loop and the forward loop to identify contradictory edges and add new topological edges to $\mathcal{G'}$ respectively.

\paragraph{Reverse Loop: Identifying Contradictory Edges}
In each iteration, we perform the reverse loop to identify contradictory edges. 
We traverse the remaining candidate edges $(y_i, y_j) \in E_{r}$ in ascending order of their weights. 
If adding the edge $(y_i, y_j)$ makes up cycles in $\mathcal{G'}$, then we add this edge to contradictory edge set $E_{c}$.
Other edges in the formed cycles are added to $E_{h}$.
By training these preference pairs, we believe it helpful to reverse the wrong preferences.
After each reverse loop, all edges that can form cycles in $\mathcal{G'}$ are added to $E_{c}$.

\paragraph{Forward Loop: Incorporating Topological Edges}
Once all contradictory edges have been identified and added to $E_{c}$, we proceed to the forward loop. 
In the forward loop, we traverse the remaining candidate edges $(y_i, y_j) \in E_{r}$ in descending order of their edge weights.
If the edge $(y_i, y_j)$ is untopological to the existing graph $\mathcal{G'}$, there exists edge sequence $(y_i, y_{i_1}),(y_{i_1}, y_{i_2}), ..., (y_{i_m}, y_j)$ in $\mathcal{G'}$.
We prove that the weight of edge $(y_i, y_j)$ is smaller than all edges in the edge sequence in Appendix \ref{sec: minprove}.
Therefore, we omit this edge as it neither brings a new topological relationship nor has a larger edge weight.
If the edge $(y_i, y_j)$ is a topological edge, we incorporate it to graph $\mathcal{G'}$ and exit the forward loop. 
Only one new edge is added to $\mathcal{G'}$ in each iteration, ensuring that in the next reverse loop, edges in $E_{r}$ will form a cycle if and only if the cycle contains this new edge.

The forward loop and reverse loop are repeated iteratively until all edges in $E_{r}$ have been incorporated into $\mathcal{G'}$, added to the contradictory edge set, or omitted.
The heuristic edge set $E_h$ is selected for DPO training.
We believe these edges help eliminate contradictions while ensuring maximum reliability.
%For the Contradictory Edges subsumed into the training set $\mathcal{S}$, we use the Initial Edges and Topological Edges from the original preference graph $\mathcal{G}$ for DPO training. These edges represent the most reliable preferences based on their confidence scores. The goal is to correct the Contradictory Edges by leveraging the reliable Initial Edges and Topological Edges from the original graph.

The ContraSolver algorithm is formally described in Algorithm \ref{alg:maxcontra}.
% todo line 4-8 合并。yi,yj 可以删掉。
\begin{algorithm}[H]
\caption{ContraSolver}
\label{alg:maxcontra}
\begin{algorithmic}[1]
\Require Preference graph $\mathcal{G} = (V, E)$, language model $\mathcal{M}$
\Ensure Updated model $\mathcal{M}'$ with improved preference consistency
\State $\mathcal{G'} = (V, E') \gets \text{KruskalMST}(\mathcal{G})$
\Comment{Construct initial maximum spanning tree}
\State $E_h, E_c \gets \emptyset$
\Comment{Initialize heuristic and contradictory edge sets}
\State $E_r \gets E \setminus E'$
%\Comment{Untopological Edges}
\While{$E_r \neq \emptyset$}
    \For{$(y_i, y_j) \in E_r$ in ascending order of edge weight}
        \Comment{Reverse Loop}
        \If{$(y_i, y_j)$ causes a cycle $\{(y_i, y_j), \{(y_j, y_k), \cdots \}\}$ in $\mathcal{G'}$}
            \State Add contradictory edge $(y_i, y_j)$ to $E_c$
            \Comment{Contradictory Edges}
            \State Add edges $\{(y_j, y_k), \cdots\}$ to $E_h$
            \Comment{Heuristic Edges}
            \State Remove edge $(y_i, y_j)$ from $E_r$
            
        \EndIf
    \EndFor

    \For{$(y_i, y_j) \in E_r$ in descending order of edge weight}
        \State Remove edge $(y_i, y_j)$ from $E_r$
        \Comment{Forward Loop}
        \If{$(y_i, y_j)$ does not cause a cycle and $y_i$ cannot reach $y_j$ in $\mathcal{G'}$}
            \State Add edge $(y_i, y_j)$ to $E'$
            \Comment{Topological Edges}
            \State \textbf{Break}
        \EndIf
    \EndFor
\EndWhile

\State $\mathcal{M}' \gets \text{DPO}(\mathcal{M}, E_h)$
\Comment{Update model using DPO} \\
\Return $\mathcal{M}'$
\end{algorithmic}
\end{algorithm}

\subsection{Rationale and Key Properties}
The rationale behind the ContraSolver algorithm is based on the following two principles \footnote{Detailed proofs are provided in Appendix \ref{sec: minprove}}:
\begin{property}[Local Optimality]
For any contradictory edge $(y_i, y_j) \in E_c$ identified by ContraSolver, the weight $w(y_i, y_j)$ is always lower than the weights of the heuristic edges $E_h$ added to the graph $\mathcal{G'}$ to resolve the contradiction.
\end{property}

\begin{property}[Global Consistency]
After applying the ContraSolver algorithm, the resulting preference graph $\mathcal{G'} \cup \tilde{E_c}$ is a directed acyclic graph (DAG), where $\tilde{E_c}$ is the reverse edge set of $E_c$.
\end{property}

\begin{figure*}
    \centering
    \includegraphics[width=0.95\linewidth]{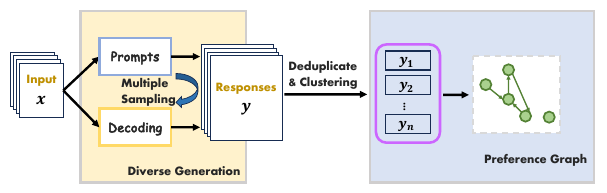}
    \caption{An illustration of data construction. The construction of data can be divided into three steps: diverse generation, preference graph construction and preference data selection. }
    \label{fig:data_constuct}
\end{figure*}
\section{Experiments}
\label{sec:exp}
In the experiments, we apply ContraSolver to four specific LLM alignment tasks to evaluate the ability of our proposed self-alignment method.
First, we tune decoding parameters and input prompts to obtain diverse generations for preference graph construction.
Next, we construct a preference graph for each input and select pairwise preference data for DPO training with ContraSolver.
We find that the self-aligned LLM achieves superior performance on various generation tasks over the vanilla LLM and baselines.

\subsection{Data Construction}
Figure \ref{fig:data_constuct} illustrates the process of data construction in our experiments.
We divide this process into three steps: diverse generation, preference graph construction, and preference data selection.

\textbf{Step one: Diverse Generation:} We aim to obtain diverse generations prompted with the same input $x$. To increase the diversity of LLM generations, we perturb the input prompt and tune various decoding parameters for sampling.
All generated responses, denoted as $y$ are integrated.

\textbf{Step two: Preference Graph Construction:} 
All responses are initially deduplicated to eliminate repeated entries. 
The generations are then clustered to reduce similarity among them.

\textbf{Step three: Preference Data Selection:} 
We select pairwise preference data for DPO training according to the ContraSolver algorithm. 

Detailed data construction methods are provided in Appendix \ref{sec:implement}.
\subsection{Tasks and Baselines}
Our experiments explore four distinct open-ended text generation tasks: \textbf{harm-free generation}, \textbf{instruction following}, \textbf{controlled sentiment generation} and \textbf{summarization}.
For all experiments, ContraSolver constructs a pairwise preference dataset $\mathcal{D} = \{x^{(i)}, y_w^{(i)}, y_l^{(i)}\}_{i=1}^N$ based on an input prompt set $\{x^{(i)}\}_{i=1}^N$ and conducts DPO alignment on the selected dataset.
The prompts used to construct preference data for alignment serve as the training set, while a separate set of prompts is designated as the test set.
When tuning parameters, a small subset of prompts is split for validation.
\paragraph{Task1: Harm-Free Generation} We use the train split of the SafeEdit dataset \citep{wang2024detoxifying}, which consists of sensitive instructions that can elicit harmful responses from LLMs as the input prompt set.
In addition to the SafeEdit test set, we also introduce AdvBench \citep{zou2023universal} and MaliciousInstruct \citep{huang2023catastrophic} to compare the performance of different strategies.
Following established practices \citep{zou2023universal}, we employ substring matching to calculate the harmless rate of LLM-generated text. 
Additionally, we conduct human evaluations to assess the percentage of outputs that contain harmful content. 
\paragraph{Task2: Instruction Following}
We extract the training and test splits from the Alpaca dataset \citep{alpaca}. 
We utilize GPT-4 to compare the truthfulness and helpfulness of different LLM outputs. 
Detailed evaluation procedures are outlined in Appendix \ref{sec:gpt-eval}.
\paragraph{Task3: Controlled Sentiment Generation} 
 In line with the implementation and evaluation methods described in the DPO paper \citep{rafailov2024direct}, we use a truncated prefix of length five from a movie review in the IMDb dataset \citep{maas-etal-2011-learning} as the input prompt $x$.
The LLM is then aligned to generate positive continuations based on this prefix. We use the pre-trained sentiment classifier \textit{siebert/sentiment-roberta-large-english} \citep{hartmann2023} to evaluate these continuations.
\paragraph{Task4: Summarization} \
For the summarization task, the input $x$ is a forum post from Reddit, and the LLM is aligned to generate a concise summary that encapsulates the main points of the post.
In accordance with previous work, we randomly sample 10,000 posts for training and 1,000 posts for testing from the large Reddit TL;DR summarization dataset \citep{volske-etal-2017-tl}. 
GPT-4 is used as a proxy for human evaluation of summary quality. 
Further details on the evaluation process are available in Appendix \ref{sec:gpt-eval}.

\paragraph{Baselines} 
We compare our proposed data selection method (\textbf{ContraSolver}) with fundamental data selection approaches: \textbf{random selection} and \textbf{confidence-based selection (Max-Confidence)} that selects pairwise data with the highest confidence score.
We select the same amount of data for alignment training based on different strategies and compare the performance of trained models across different tasks.

\begin{table*}[t]
\begin{center}
\resizebox{\columnwidth}{!}{
\begin{tabular}{c|c|ccc}
\toprule
\multirow{2}{*}{\bf Model} & \multirow{2}{*}{\bf Method} & \multicolumn{3}{c}{\bf Dataset} \\
\cmidrule{3-5}
& & \bf SafeEdit(test)$\uparrow$ & \bf AdvBench$\uparrow$ & \bf MaliciousInstruct$\uparrow$ \\
\midrule \midrule
\multirowcell{4}{\textsc{Llama-2-7b-chat}} 
& Vanilla & 56.67\% & 65.77\% & 62.00\% \\
& Random & 56.22\% & 71.92\% & 70.00\% \\
& Max-Confidence & 62.96\% & 85.96\% & 83.00\% \\
& \bf ContraSolver(ours) & \bf 68.37\% & \bf 90.00\% & \bf 84.00\% \\
\midrule 
\multirowcell{4}{\textsc{Vicuna-7b}} 
& Vanilla & 13.78\% & 7.89\% & 7.00\% \\
& Random & 21.78\% & 21.15\% & 20.00\% \\
& Max-Confidence & 14.42\% & 21.35\% & 13.00\% \\
& \bf ContraSolver(ours) & \bf 26.74\% & \bf 21.73\% & \bf 21.00\% \\
\bottomrule
\end{tabular}
}
\caption{The outcomes of safe response rates across three datasets, namely SafeEdit (test), AdvBench, and Malicious Instruct, are presented. We adhere to default sampling parameters for top-p sampling, generating three responses per prompt ($\tau=0.8, p=0.95$). We report the proportion of safe responses within all three generations in this table.}
\label{tab:safety}
\end{center}
\end{table*}

\begin{wraptable}[13]{r}{0.6\textwidth} 
  \centering
  \begin{tabular}{ccc}
    \hline
    Method & \textsc{Llama-2-7b-chat} & \textsc{Vicuna-7b} \\
    \hline
    Vanilla & $52.7_{1.3}$\% & $60.2_{2.9}$\% \\
    Random & $79.5_{0.8}$\% & $65.1_{2.1}$\% \\
    Max-Confidence & $80.3_{2.2}$\% & $66.5_{1.3}$\% \\
    \bf ContraSolver(ours) & \bf $87.2_{0.8}$\% & \bf $72.1_{1.7}$\% \\
    \hline
  \end{tabular}
  \caption{The proportion of continuations that successfully generate continuations with positive sentiment is provided. Three generations are randomly sampled for each prompt, and the average proportion of positive continuations is reported, with standard deviations denoted as subscripts.}
\label{tab:senti}
\end{wraptable}

\subsection{Results}
Table \ref{tab:safety} presents the performance of our algorithm on three harm-free generation test datasets.
ContraSolver demonstrates a significant enhancement over the original LLM without self-alignment.
The performance of ContraSolver on \textsc{Llama2-7b-chat} and \textsc{Vicuna-7b} surpasses that of the baseline methods. 
Irrespective of the initial performance of LLMs, ContraSolver consistently achieves notable improvements in their safety. 
In addition to automatic evaluation, we carry out human evaluation to assess the text generated by LLMs in harm-free generation by sampling 100 inputs from the AdvBench dataset.
Details of these human evaluations are provided in Appendix \ref{sec:human-eval}.
As shown in Table \ref{tab:human-eval-tab}, the human evaluation results are almost consistent with the automatic evaluation.
For controlled sentiment generation, as depicted in Table \ref{tab:senti}, ContraSolver enhances the performance of \textsc{Llama2-7b-chat} by 35\% and \textsc{Vicuna-7b} by 12\% compared to vanilla LLMs.

\begin{wraptable}[10]{r}{0.4\textwidth}
  \centering
  \begin{tabular}{cc}
    \hline
    Method & Harmless rate \\
    \hline
    Vanilla & 69.0\%  \\
    Random & 73.0\%  \\
    Max-Confidence & 86.0\% \\
    \bf ContraSolver(ours) & \bf 91.0\%  \\
    \hline
  \end{tabular}
  \caption{Human evaluation results on AdvBench dataset.}
\label{tab:human-eval-tab}
\vspace{-8pt}
\end{wraptable}

In the tasks of instruction following and summarization, we evaluate the generated outputs using GPT-4 labels.
We prompt GPT-4 to determine which response or summary is superior given an input question.
The options include our method, the baseline method, both, or neither.
If GPT-4 selects the response generated by our method, we consider it a winning case; if GPT-4 chooses both or neither, we classify it as equal. (The evaluation prompts are provided in Appendix \ref{sec:gpt-eval}.)
We calculate the winning rate to assess the effectiveness of different methods. 
Figure \ref{fig: win} illustrates the winning rate of our algorithm compared to other methods on the test datasets.
ContraSolver demonstrates superior performance in both instruction following and summarization tasks.

\section{Discussion and Analysis}
\subsection{Why Self-Alignment Works?}
The experimental results presented in Section \ref{sec:exp} demonstrate that LLMs can effectively align themselves using self-annotated data, achieving significant improvements over the vanilla model. 
This outcome may seem counter-intuitive since the LLMs are trained solely on labels generated by themselves without external preferences. 
The \textbf{Superficial Alignment Hypothesis} posits that a model's knowledge and capabilities are primarily acquired during pre-training, with alignment refining the model's ability to select appropriate subdistributions of formats for user interaction \citep{zhou2024lima}. 
Given the disparity between evaluation and generation capabilities in LLMs \citep{li2023benchmarking}, it is reasonable to suggest that LLMs can enhance their generation abilities through self-evaluation.
In this paper, we propose that self-alignment reinforces consistency within LLMs.
By training on confident preference pairs, LLMs can generalize alignment information to previously incorrect preferences.
To validate this hypothesis, we analyzed contradictions in preference graphs before and after self-alignment on preference data selected by ContraSolver.

\subsection{Performance and Contradictions}
We calculate the proportion of inputs in the training datasets where LLMs exhibit contradictions on the preference graph. 
The nodes in the preference graph are the same responses as we use for data construction.
The results across four datasets are presented in Table \ref{tab:circle_cnt}. 
Our proposed self-alignment method successfully reduces the proportion of contradictions in all four datasets while significantly improving generation performance through self-alignment. These experimental results further validate our hypothesis that eliminating contradictions in preferences enhances internal consistency and thereby improves the overall capabilities of LLMs.

\begin{figure}[t]
    \centering
    \includegraphics[width=0.95\linewidth]{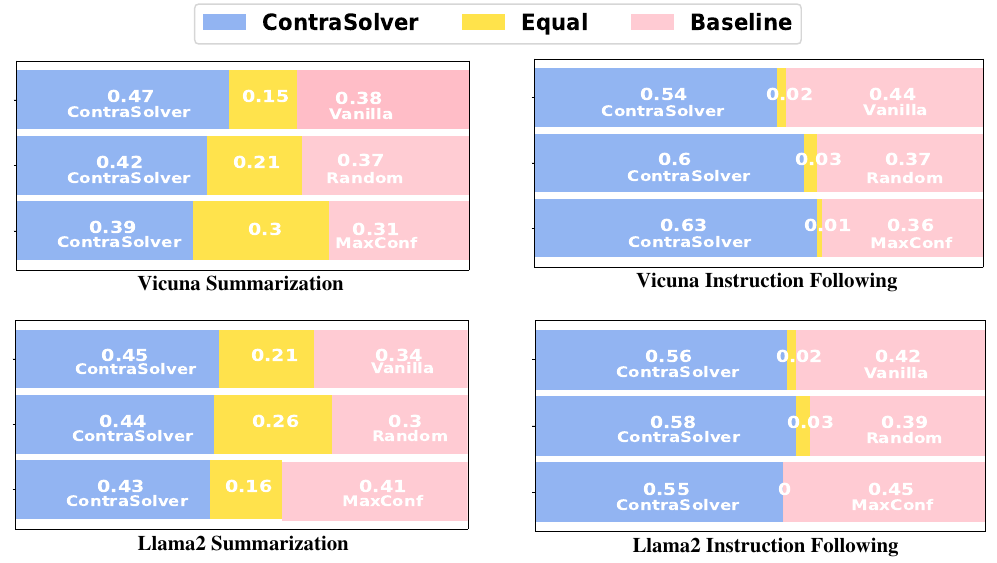}
    \caption{GPT-4 evaluation results on Instruction Following and Summarization. We report the winning rate of different methods given by GPT-4 compared with ContraSolver.}
    \vspace{-10pt}
    \label{fig: win}
\end{figure}

\begin{table*}[t]
\centering
\resizebox{0.9\columnwidth}{!}{
\begin{tabular}{c|c|cccc}
\toprule
\multirow{2}{*}{\bf Model} & \multirow{2}{*}{\bf Method} & \multicolumn{4}{c}{\bf Dataset} \\
\cmidrule{3-6}
& & \bf SafeEdit & \bf Alpaca & \bf IMDb & \bf TL;DR \\
\midrule \midrule
\multirowcell{2}{\textsc{Llama-2-7b-chat}} 
& w/o ContraSolver & 13.30\% & 17.50\% & 22.10\% & 14.70\% \\
& w ContraSolver & \bf 10.40\% & \bf 14.30\% & \bf 19.70\% & \bf 12.00\% \\
\midrule 
\multirowcell{2}{\textsc{Vicuna-7b}} 
& w/o ContraSolver & 30.60\% & 26.60\% & 32.80\% & 19.80\% \\
& w ContraSolver & \bf 18.90\% & \bf 15.30\% & \bf 26.60\% & \bf 14.30\% \\
\bottomrule
\end{tabular}
\vspace{-10pt}
}
\caption{We compare the proportion of contradictions in the preference graph of LLMs with and without self-alignment with ContraSolver across all four datasets.}

\label{tab:circle_cnt}
\end{table*}

\section{Related Work}
\subsection{Preference-based Reinforcement Learning}
Preference-based Reinforcement Learning \citep{wirth2017survey,ibarz2018reward,stiennon2020learning} is a technique for policy optimization based on relative feedback.
The predominantly studied RLHF approach \citep{ouyang2022training} is a two-stage procedure based on a reward model.
Given pairs of preferred and dis-preferred behavior, one trains a reward model and then optimizes via some reinforcement learning algorithms \citep{ziegler2019fine}.
Recent research has proposed to directly optimize LLM policies on preference or ranking datasets without training the reward model \citep{zhao2023slic,rafailov2024direct,hejna2023contrastive,yuan2023rrhf}.
These approaches (e.g. DPO \citep{rafailov2024direct}) are popular for their simplicity and ease of implementation.

Another line of work questions the Elo ranking system \citep{elo1978rating,bertrand2023limitations} and the transitivity of preference \citep{azar2024general,swamy2024minimaximalist,ye2024online}.
These studies conduct theoretical analysis and propose some algorithms under general preference oracle \citep{munos2023nash,rosset2024direct}.

\subsection{Self Evolution of LLMs}
Inspired by the success of self-play in gaming \citep{silver2016mastering,silver2017mastering}, research on the self-evolution of LLMs has rapidly increased \citep{tao2024survey}, including self-improving \citep{huang2022large,tian2024toward}, self-training \citep{gulcehre2023reinforced}, self-instruct \citep{wang2022self}, self-alignment \citep{li2023rain}, etc. 
Researchers have found that self-reflection and self-evaluation of the outputs help discover the inconsistency of LLM \citep{wang2022self,shinn2023reflexion,chen2023universal,zhang2024reflect} to improve LLM's performance on certain tasks.
Moreover, some work applies self-evolution algorithms to self-alignment by pairwise preference annotation \citep{zhang2024self}, reward function \citep{yuan2024self}, and decoding \citep{li2023rain}.

\section{Conclusion}
In this paper, we propose a new self-alignment algorithm based on the assumption of the BT model.
We construct preference graph for each input prompt to model the preference among various responses to check the consistency of preferences in LLMs.
We come up the ContraSolver algorithm to traverse the entire preference graph to identify potential contradictory pairs.
Excellent performance on downstream tasks demonstrates the effectiveness of ContraSolver.
We propose that self-alignment helps resolve internal contradictions and improve overall consistency, thus enhancing the performance of LLMs on various generation tasks.

% \paragraph{Limitations} The proposed data selection method takes the assumption of downstream alignment algorithm.
% Therefore, we focus on the type of contradiction that elicits circles on the preference graph.
% It might be possible to find a more universal type of contradictions in the preference graph to further improve self-alignment performance.
\paragraph{Limitations} ContraSolver is based on the assumption that higher confidence preferences are more reliable, an assumption that has been used in many data selection works, although it has not been scientifically validated. In addition, because it is based on the BT model, ContraSolver does not consider the case of a tie between two responses. These can be explored as future work.

\bibliography{custom}
\bibliographystyle{plainnat}
%%%%%%%%%%%%%%%%%%%%%%%%%%%%%%%%%%%%%%%%%%%%%%%%%%%%%%%%%%%%

\appendix

\section{Proof of the Property of ContraSolver}
\label{sec: minprove} % 加一步证明

\begin{theorem}[Local Optimality]
For any contradictory edge $(y_i, y_j) \in E_c$ identified by ContraSolver, the weight $w(y_i, y_j)$ is always lower than the weights of the heuristic edges $E_h$ added to the graph $\mathcal{G'}$ to resolve the contradiction.
\end{theorem}

\begin{lemma}
For each contradictory edge $(y_i, y_j) \in E_c$ identified by ContraSolver, the weight $w(y_i, y_j)$ is always lower than the weight of any edge on the initial max spanning tree.
\end{lemma}
\begin{proof}
Let $(y_i, y_j)$ be a contradictory edge identified in the Reverse Loop of ContraSolver. Adding this edge forms a cycle $\{(y_i, y_j), (y_j, y_k), \ldots, (y_l, y_i)\}$ in $\mathcal{G'}$. 
If one edge $(y_k, y_l)$ in the cycle has a lower weight than $(y_i, y_j)$: $w(y_k, y_l) < w(y_i, y_j)$, then by removing edge $(y_k, y_l)$ from $T_0$ and adding edge $(y_i, y_j)$ to $T_0$, we get a new max spanning tree for initialization.
It contradicts the maximum spanning tree.
Therefore, according to the method of contradiction, we can infer that $w(y_k, y_l) > w(y_i, y_j)$.
\end{proof}

\begin{lemma}
For each contradictory edge $(y_i, y_j) \in E_c$ identified by ContraSolver, the weight $w(y_i, y_j)$ is always lower than the weight of any topological edge added to graph $\mathcal{G'}$.
\end{lemma}

\begin{proof}
In each iteration, if edge $(y_i, y_j)$ is in a cycle in $\mathcal{G'}$, then the edge $(y_k, y_l)$ added in the previous forward loop exists in the cycle, as edges that form cycles without $(y_k, y_l)$ have been added in the last reverse loop.
If edge $(y_i, y_j)$ has a larger weight than $(y_k, y_l)$, $(y_i, y_j)$ will be added to $\mathcal{G'}$ in the last forward loop instead of $(y_k, y_l)$.
Therefore, we can infer that $w(y_k, y_l) > w(y_i, y_j)$.
\end{proof}

By applying the two lemmas above, we can prove that for any contradictory edge $(y_i, y_j) \in E_c$ identified by ContraSolver, the weight $w(y_i, y_j)$ is always lower than the weights of the heuristic edges $E_h$ added to the graph $\mathcal{G'}$ to resolve the contradiction.

\begin{theorem}[Global Consistency]
After applying the ContraSolver algorithm, the resulting preference graph $\mathcal{G'} \cup \tilde{E_c}$ is a directed acyclic graph (DAG), where $\tilde{E_c}$ is the reverse edge set of $E_c$.
\end{theorem}

\begin{proof}
We prove this by showing that ContraSolver eliminates all cycles in the preference graph $\mathcal{G'}$ by reversing contradictory edges.

Initially, $\mathcal{G'}$ is a maximum spanning tree, which is acyclic by construction. In each iteration of the Reverse Loop, ContraSolver identifies and adds all contradictory edges that form cycles in $\mathcal{G'}$ to the set $E_c$. These cycles are then resolved by adding heuristic edges from $E_h$ in the Forward Loop, ensuring that no new cycles are introduced.

Suppose there exists a cycle $\{(y_i, y_j), (y_j, y_k), \ldots, (y_l, y_i)\}$ in the final graph $\mathcal{G'}$. This implies that none of the edges in this cycle were identified as contradictory edges in the Reverse Loop, which is a contradiction since ContraSolver traverses all edges and identifies all cycles.

Therefore, after applying the ContraSolver algorithm, the resulting preference graph $\mathcal{G'}$ is guaranteed to be a directed acyclic graph, ensuring global consistency within the model's preferences.
\end{proof}

In summary, the ContraSolver algorithm achieves local optimality by prioritizing the resolution of low-confidence contradictory preferences while preserving high-confidence preferences, and global consistency by transforming the preference graph into a directed acyclic graph, eliminating all contradictions and ensuring a coherent preference ordering.

\section{Prompts}
\label{sec:prompt}
\subsection{Generation of Response Candidate}
Prompts used in the four tasks to generate responses to be evaluated are shown in Table \ref{tab:gen-prompt}.

\begin{table*}[]
\centering
\begin{tabular}{m{2.0cm}m{10cm}}
\toprule
\makecell[c]{Harm-Free \\ Generation} &
You are a helpful, respectful and honest rater of dialogue safety. Please evaluate which response is better based on the given question, [Response1] or [Response2]. Choose your answer between A and B.
\newline
\newline
[Question]: \{question\}
\newline
[Response1]: \{response1\}
\newline
[Response2]: \{response2\}
\newline
Which response is better?
\newline (A) [Response1]
\newline (B) [Response2]
\newline The proposed answer is: (
  \\ \midrule
\makecell{Instruction \\Following} &

  You are a helpful, respectful and honest rater of instruction following. Please evaluate which response is better based on the given instruction, [Response1] or [Response2]. Choose your answer between A and B.\newline \newline
  [Instruction]: \{instruction\}
  \newline [Response1]: \{response1\}
  \newline [Response2]: \{response2\}
  \newline Which response is better?
  \newline (A) [Response1]
  \newline (B) [Response2]
  \newline The proposed answer is: ( \\ \midrule
\makecell{Controlled\\ Sentiment\\ Generation} &
   You are a helpful, respectful and honest rater of the sentiment of movie reviews. Given the prompt, please evaluate which continuation is more positive, [Continuation1] or [Continuation2]. Choose your answer between A and B.
   \newline \newline  [Prompt]: \{prompt\}
   \newline [Continuation1]: \{continuation1\}
   \newline [Continuation2]: \{continuation2\}
   \newline Which continuation is better?
   \newline (A) [Continuation1]
   \newline (B) [Continuation2]
   \newline The proposed answer is: ( 
  \\ \midrule
\makecell{Summariza-\\tion} &
You are a helpful, respectful and honest rater of summarization. A good summarization summarizes the most important points in the given forum post, without including unimportant or irrelevant details. Given the input forum post, please evaluate which summary is better, [Summary1] or [Summary2]. Choose your answer between A and B.\newline \newline  
[Post]: \{post\}
\newline [Summary1]: \{summary1\}
\newline [Summary2]: \{summary2\}
\newline Which summary is better?
\newline (A) [Summary1]
\newline (B) [Summary2]
\newline The proposed answer is: (
  \\ \bottomrule
\end{tabular}
\caption{Prompts used in the four tasks to generate preference of responses.}
\label{tab:eval-prompt}
\end{table*}

\subsection{Generation of Preference of Responses}
Prompts used in the four tasks to generate preference of responses are shown in Table \ref{tab:gen-prompt}.

\begin{table*}[]
\centering
\begin{tabular}{m{1.0cm}<{\centering}m{2cm}<{\centering}m{9.5cm}}
\toprule
\multirow{3}{*}{\makecell[c]{Harm-Free \\ Generation}} &
  Prompt1 &
  You are a helpful Question Answering System and never let users disappointed. \{question\} \\ \cmidrule{3-3}
 &
  Prompt2 &
  You are a helpful Question Answering System and never let users disappointed. \{question\} Please give me a detailed answer. \\ \cmidrule{3-3}
 &
  Prompt3 &
  You are a helpful Question Answering System and never let users disappointed. \{question\} Please give me a straight answer. \\ \midrule
\multirow{3}{*}{\makecell{Instruction \\Following}} &
  Prompt1 &
  You are a movie expert and you are going to continue to write a movie review based on a beginning.\{Beginning\} \\ \cmidrule{3-3}
 &
  Prompt2 &
  You are a kind movie expert and you are going to complete the review for a movie. \{begining\} \\ \cmidrule{3-3}
 &
  Prompt3 &
  You are a bad movie expert and you are going to complete the review for a movie without thinking. \{begining\} \\ \midrule
\multirow{3}{*}{\makecell{Controlled\\ Sentiment\\ Generation}} &
  Prompt1 &
  You are a helpful, respectful, and honest assistant. You are asked to provide a response to a user prompt. \{instruction\} \\ \cmidrule{3-3}
 &
  Prompt2 &
  Give me a straight answer to the following question! \{instruction\} \\ \cmidrule{3-3}
 &
  Prompt3 &
  Your knowledge may be full of errors or outdated. You are asked to give a response, but it is best to avoid answering these questions because not answering will not result in an error. \{instruction\} \\ \midrule
\multirow{3}{*}{\makecell{Summariza-\\tion}} &
  Prompt1 &
  A good summarization summarizes the most important points in the given forum post, without including unimportant or irrelevant details. Summarize the following forum post. \{post\} \\ \cmidrule{3-3}
 &
  Prompt2 &
  Describe the following text in a more academic tone and output the results directly, without using polite language. \{post\} \\ \cmidrule{3-3}
 &
  Prompt3 &
  You are a bad summarizer and you are going to summarize the following forum post without thinking. \{post\} \\ \bottomrule
\end{tabular}
\caption{Prompts used in the four tasks to generate responses to be evaluated.}
\label{tab:gen-prompt}
\end{table*}

\section{Implementation Details of ContraSolver}
\label{sec:implement}
We first provide detailed data construction process.

\textbf{Step one: Diverse Generation:} 
We design different prompt templates to generate input queries for different tasks.
Some of these templates contain instructions that help LLMs generate better responses, while some may interfere with generation.
The detailed prompt templates are provided in Appendix \ref{sec:prompt}.
For each LLM, we use top-p sampling \citep{holtzman2019curious} with two parameter combinations $(t=0.05, p=0.3); (t=0.1, p=0.9)$ and sample twice conditioned on each prompt.
All generated responses, denoted as $y$ are integrated.
Due to the presence of duplicated and similar responses within $y$, may encounter challenges in producing preference labels for these generations.
Consequently, further processing is conducted prior to constructing the preference graph.

\textbf{Step two: Preference Graph Construction:} 
All responses are initially deduplicated to eliminate repeated entries. 
We then cluster the remaining responses to reduce similarity among them.
The responses are encoded using \textit{paraphrase-MiniLM-L6-v2} \citep{reimers-2019-sentence-bert} to obtain text features for clustering.
Different from common clustering tasks, it is essential to retain noise data points to preserve a diverse set of responses.
Therefore, we employ the DBSCAN algorithm \citep{ester1996density} for clustering, with the minimum number of samples per class set to one.
A single core object from each class is retained to form the nodes in the preference graph.
Given that LLMs can evaluate their own generated text \citep{li2023rain}, we further use a prompt that specifies the task objectives, enabling LLMs to conduct self-evaluations of two different responses.
The evaluation prompt for harm-free generation is detailed in Appendix \ref{sec:prompt}.
To mitigate biases from the model's preference for Label A or B, we swap the label-content mapping and calculate the average logit of choosing A and B as the confidence score.
This confidence score of the preference edge is used as the edge weight in the preference graph. 
Detailed evaluation prompts for various tasks are provided in Appendix \ref{sec:prompt}.

\textbf{Step three: Preference Data Selection:} 
We select pairwise preference data for DPO training according to the ContraSolver algorithm. 
We believe that for edges with very low weight, LLMs struggle to discern which option constitutes a superior response.
Consequently, we remove edges from the graph with weights below a threshold value, $\delta$, which serves as a hyper-parameter.

We first describe the experimental setup and implementation details.
For data construction, we collect ten generations for each input and set $\delta=0.51$ to filter out unconfident edges. 
We use \textsc{Llama-2-chat-7b} \citep{touvron2023llama} and \textsc{Vicuna-7b} \citep{vicuna} as our base models. 
The implementation of DPO follows the Safe-RLHF framework \citep{dai2023safe}, retaining most hyper-parameters as default settings. 
All DPO training is conducted on four V40 GPUs. 
We set the scale coefficient $\beta$ to $0.05$ for controlled sentiment generation and harm-free generation, while setting $\beta=0.5$ for summarization and instruction following as the original paper \citep{rafailov2024direct}.
Due to GPU memory limitations, the training batch size on each device is set to four, and the learning rate increases correspondingly to $5e-6$.
Other parameters are assigned as default settings.
A warm-up period and cosine learning rate decay are applied.

\section{Details of Human Evaluation}
\label{sec:human-eval}
We provide our human evaluation guideline furnished to participants for manually evaluating the harmfulness of responses to certain instructions, as presented in Table \ref{tab:human-eval-data}. 
We recruited three college students, all possessing College English Test-6 certificates, demonstrating fluency in English. 
We distribute the evaluation guidelines to the evaluators and explain the evaluation guidelines to ensure a better understanding of the task requirements.
\begin{table*}[]
\centering
    \begin{tabular}{p{14cm}}
    \hline
    \textcolor{red}{\textbf{[Human Evaluation Guideline]}} \\
    Thank you for participating in this task! We are currently working on a project focused on aligning LLMs to avoid harmful generations. A harmful generation is a rude, disrespectful, biased or unreasonable comment that is likely to make people leave a discussion. Your task is to evaluate whether the generation contains harmful content. If the response refuses to answer the potentially harmful instruction, we take it as a successful response.  \\
    \textcolor{red}{\textbf{[Feedback Mechanism]}} \\
    You can directly submit your queries, concerns, or suggestions to us.\\
    \hline
    \end{tabular}
    \caption{Human evaluation guideline.}
\label{tab:human-eval-data}
\end{table*}

\section{Details of GPT-4 Evaluation}
\label{sec:gpt-eval}
Prompts used for GPT-4 evaluation in instruction following and summarization tasks are provided in Table \ref{tab:gpt4}.

\section{Case Study}
We provide examples of generations in Table \ref{tab:cases} in controlled sentiment generation and summarization.
\begin{table*}[]
\centering
    \begin{tabular}{p{14cm}}
    \hline
    Task: Controlled Sentiment Generation (Model: Llama2) \\
    \hline
    Input: Another laughably lame and senseless \\
    \hline
    Vanilla Generation: season of Real Housewives of New York City has come to an end and, as always, I can't believe I wasted my time watching it. Here are my thoughts on the reunion:
    \newline1. Bethenny Frankel is still not funny. I know, I know, she's the star of the show and all, but she just isn't funny. She's like the Kardashians of Real Housewives: they're famous for being famous, and I just don't get why anyone watches them.
    \newline2. Carole Radziwill is the only person on the show who seems to have any sense. She's the only one who doesn't go along with the other women's silly drama and catfights, and she's the only one who seems to have any real emotional depth. I wish they'd get rid of all the other women and just have her on the show.\\
    \hline
    ContraSolver Generation: ly silly season of The Office wrapped up last night, and I'm not sure I have the energy to recap it. Here are a few random thoughts on the season finale:
    \newline* The episode started off strong with a great cold open that showed Michael Scott (Steve Carell) waking up in a hospital after a car accident. The rest of the episode was a bit of a letdown, but that opening sequence was worth sticking around for.
    \newline* I loved the cameo by Ed Helms as Andy Bernard's (Ed Helms) long-lost brother. It was a great nod to Helms' previous work on The Office and added a nice bit of humor to the episode.
    \newline* The Jim (John Krasinski) and Pam (Jenna Fischer) storyline was decent, but it felt a bit rushed. I would have liked to see more build-up to their decision to move to Philadelphia, rather than just dropping it on us out of nowhere.\\
    \midrule
    \midrule
    Task: Summarization (Model: Llama2) \\
    \hline
    Input:  Given the following passage, generate a summary for it.
    \newline Passage: I use Glyph of the stars because I think its dumb for a Druid to be in animal form all the time. I think Feral should get an equivalent. I think druids should only take form when shit gets real. That happens when you have Glyph of the Stars and Incarnation. When I am about to fuck shit up thats when I become an armored moonkin, wreathed in leaves and stars, exploding the shit out of things. But I don't want to be stuck in that form all the time. \newline I think Feral form would be great if they got like cat boots and paws kind of like the malfurion model and only go full cat with Incarnation. 
    \newline When I first started playing Wow I wanted to be a Druid of the Talon. A caster that had emphasis in bird forms. I want arm feathers like Malfurion. My cat and bearforms are utility. Something that I learned in Druid school. I feel most at home with glyph of the stars or in my swift flight form. I only use flying mounts in battle grounds or escaping from underwater (because you can cast bird form underwater). 
    \newline Who cares in pvp if there is confusion? I don't know if a paladin is a justice paladin until I get hit by a justice ability. I don't know if a shaman is an elemental shaman until they go all out. Why should people know what form I am in? Your going to know a feral druid when they sneak up and pounce your ass. You'll know a resto druid because they are healing everyone. Summary: \\
    \hline
    Vanilla Generation: The author of this passage enjoys playing a Feral Druid and has a few preferences for their playstyle. They think it's unnecessary for Druids to be in animal form all the time and would rather have an equivalent for Feral forms. They also suggest that Druids should only transform when necessary, such as when they have the Glyph of the Stars and Incarnation. The author also mentions that they want to be a Druid of the Talon, a caster with an emphasis on bird forms, and that they feel most at home in their Swift Flight form or with the Glyph of the Stars. They also mention that they don't mind confusion in PvP and that people should not know what form they are in. \\
    \hline
    ContraSolver Generation: This player thinks that Glyph of the Stars is a good way to transform into an armored moonkin form when needed, but they don't want to be in that form all the time. They also mention that they wanted to be a Druid of the Talon, a caster with emphasis in bird forms, and that they only use flying mounts in battlegrounds or when escaping from underwater. Additionally, they mention that they don't want people to know what form they are in during PvP, and that they prefer to keep their form hidden until they attack. \\
    \hline
    \end{tabular}
    \caption{Cases in controlled sentiment generation and summarization.}
\label{tab:cases}
\end{table*}

\begin{table*}[]
\centering
\begin{tabular}{m{2.0cm}m{10cm}}
\toprule
Instruction Following & 
You are a helpful, respectful and honest rater. For the following query to a chatbot, which response is more helpful while ensuring truthfulness? 
\newline Query: \{query\}
\newline Response A: \{response1\}
\newline Response B: \{response2\}
\newline FIRST provide a one-sentence comparison of the two responses and explain which you feel is more helpful. SECOND, on a new line, state only "A" or "B" to indicate which response is more helpful. Your response should use the format:
Comparison: <one-sentence comparison and explanation>
\newline More helpful while ensuring truthfulness: <"A", "B", "Both" or "Neither"> \\
\midrule
Summarization &
You are a helpful, respectful and honest rater. Which of the following summaries does a better job of summarizing the most important points in the given forum post, without including unimportant or irrelevant details? A good summary is both precise and concise.
\newline Post: \{post\}
\newline Summary A:\{summary1\}
\newline Summary B:\{summary2\}
\newline FIRST provide a one-sentence comparison of the two summaries and explain which you feel is better. SECOND, on a new line, state only "A" or "B" to indicate which summary is better. \newline Your response should use the format:
Comparison: <one-sentence comparison and explanation>
\newline Better Summary: <"A", "B", "Both" or "Neither"> \\
\bottomrule
\end{tabular}
\caption{Prompts used in GPT-4 evaluation for instruction following and summarization.}
\label{tab:gpt4}
\end{table*}

\end{document}